\documentclass[letterpaper, 10 pt, conference]{ieeeconf}  

\IEEEoverridecommandlockouts                              

\usepackage{amssymb}
\usepackage{verbatim}
\usepackage{amsmath}
\usepackage{graphicx}
\usepackage{wrapfig}
\usepackage{color}
\usepackage[titlenumbered,ruled]{algorithm2e}
\usepackage{url}
\usepackage{subcaption}
\usepackage{caption}
\newtheorem{asm}{Assumption}

\newtheorem{dfn}{Definition}
\newtheorem{prp}{Proposition}
\newtheorem{crl}{Corollary}

\newtheorem{obs}{Observation}

\DeclareMathOperator{\dist}{d}
\DeclareMathOperator{\rot}{R}
\DeclareMathOperator{\dsign}{d_{s}}
\DeclareMathOperator{\D}{D}
\DeclareMathOperator{\diam}{diam}
\newcommand{\R}{\mathbb{R}}

\pdfoutput=1

\title{\LARGE \bf
A Decomposition-Based Approach to Reasoning \\ about Free Space Path-Connectivity for Rigid Objects in 2D
}

\author{Anastasiia Varava, J. Frederico Carvalho, Danica Kragic, and Florian T. Pokorny
\thanks{The authors are with the Robotics, Perception, and Learning Lab, School of Computer Science and Communication, KTH The Royal Institute of Technology, 100 44 Stockholm, Sweden 
        {\tt\small \{varava, jfpbdc, dani, fpokorny\}@kth.se}}}

\begin{document}

\maketitle
\thispagestyle{empty}
\pagestyle{empty}

\begin{abstract}
In this paper, we compute a conservative approximation of the path-connected components of the free space of a rigid object in a 2D workspace in order to solve two closely related problems: {\it (i)} to determine whether there exists a collision-free path between two given configurations, and {\it (ii)} to verify whether an object can escape arbitrarily far from its initial configuration -- i.e., whether the object is {\it caged}. Furthermore, we consider two quantitative characteristics of the free space: the volume of path-connected components and the width of narrow passages. To address these problems, we decompose the configuration space into a set of two-dimensional slices, approximate them as two-dimensional alpha-complexes,  and then study the relations between them. This significantly reduces the computational complexity compared to a direct approximation of the free space. We implement our algorithm and run experiments in a three-dimensional configuration space of a simple object showing runtime of less than 2 seconds.
\end{abstract}




\section{Introduction}
\label{intro}
Understanding the global topological and geometric properties of the free space is important for both robotic manipulation and motion planning. 

In manipulation, the mobility of an object may be constrained by manipulators and/or obstacles. Here, it is crucial to understand the object's free space to quantify how far it can move from its initial configuration. When the object cannot escape arbitrarily far, we in particular say that the object is {\it caged}. Formally, this means that it is located in a compact path-connected component of its free space. Caging can be applied to robotic grasping and multi-agent manipulation.

One of the biggest challenges in caging is verification -- i.e., designing efficient algorithms providing theoretical guarantees that a given configuration is a caging configuration. To prove that an object is located in a bounded path-connected component of its free space requires knowledge about the entire configuration space, which is high dimensional even in the case of rigid objects (three-dimensional when the workspace has only two dimensions). This makes direct reconstruction of the free space computationally expensive. Another approach towards verifying caging relies on particular geometric and topological features of the object under consideration. These algorithms can be computationally efficient, but they are limited to particular classes of objects' shapes. 

\begin{figure}[htb!]   
\label{fig::workspace}
 
\center{\includegraphics[width=0.5\textwidth]{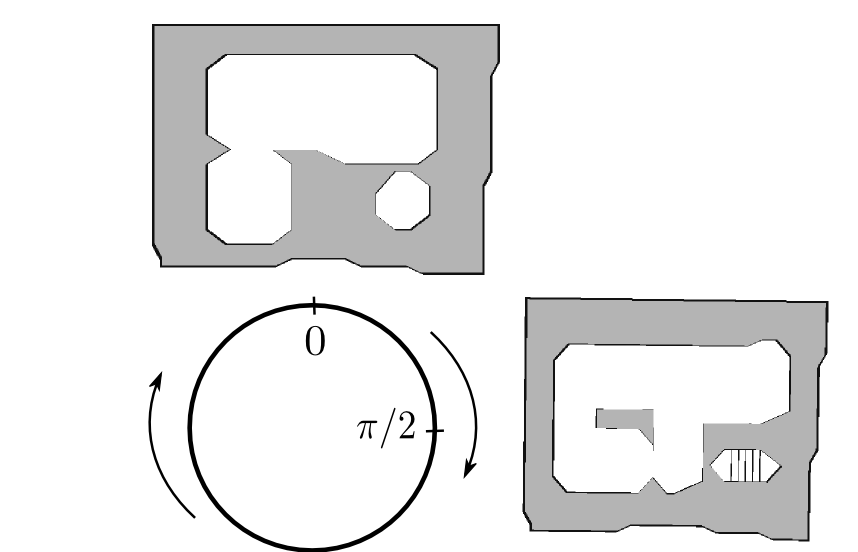}}
  \caption{We approximate the collision space of an object by  choosing a finite set of fixed object's orientations and considering the corresponding projections of the collision space to $\R^2$. This figure shows the projections corresponding to orientations $0$ and $\pi/2$. The corresponding workspace can be found in Sec.\ref{experiments}, the object is a horizontal bar.}
\label{fig::teaser}  	
\end{figure}

In contrast, the only assumption we make about the shape of the object is that it should not have `thin parts', see Sec.\ref{definitions} for details.
We also address the problem of proving path non-existence between a pair of given configurations. To prove that there is no path, one has to show that two configurations belong to different path-connected components.

Apart from the purely topological property of being located in a certain path-connected component, it is also important to quantitatively describe the amount of freedom the object possesses in a certain configuration. Indeed, assume that our task is to cage an object and move it to a certain position. The more mobility the object has within the cage, the less precision we can achieve trying to put it into the goal position. Therefore, we would prefer those caging configurations where the object's mobility is more restricted. To distinguish between different cages, we compute the volume of path-connected components.

To show that two configurations are disconnected, we construct an approximation of the object's collision space. Intuitively, our approximation is a set of projections of a subset of the object to planes with fixed orientations, see Fig.\ref{fig::teaser}. By construction, our approximation is a proper subset of the real collision space, which means that if our algorithm reports that the two configurations are disconnected, then there is no path between them in the real free space. However, for the same reason our algorithm is not guaranteed to find all the possible caging configurations, since we do not reconstruct the entire collision space.

In \cite{varava_2}, we presented the general idea of our approach. While in this paper, we present an elaborate theoretical framework, prove the correctness of our method, implement the algorithm for two-dimensional workspaces, and report the experimental results. The core contributions of the present paper with respect to \cite{varava_2} can be summarized as follows:
{\it (i)} we provide a correctness proof of our algorithm (i.e., we show that if in our approximation of the free space two configurations are disconnected, then they are disconnected in the real free space).
{\it (ii)} We provide algorithms to compute lower bounds for volume of path-connected components and width of narrow passages.
{\it (iii)} We provide implementation details and run experiments in a two-dimensional workspace. 
{\it (iv)} We prove that if two configurations are disconnected, we can construct a good enough approximation of the free space to show that these configurations are either disconnected or connected by a narrow passage.

\section{Motivation and Related Work}
\label{related_work}

 In grasping, caging can be considered as an alternative to a classical grasp \cite{makita, makita2, pokorny, varava}, as well as an intermediate step on the way towards a firm grasp \cite{rodriguez}. Unlike classical grasping, caging can be formulated as a purely geometric problem, and therefore one can derive sufficient conditions for an object to be caged. Apart from that, caging deals with global geometric features of the object, which may in applications be more robust to noise than local geometry. To prove that a rigid object is caged it is enough to prove this for any subset (part) of the object. This allows a large enough subset of the object lying strictly inside the real object to be considered. This makes caging robust to noise and uncertainties appearing in the process of shape reconstruction and position detection.
The notion of a planar cage was initially introduced by Kuperberg in 1990 \cite{kuperberg} as a set of $n$ coplanar points lying in the complement of a polygon and preventing it from escaping arbitrarily far from its initial position. In tobotics, it was subsequently studied in the context of point-based caging in 2D by Rimon and Blake \cite{rimon}, Pipattanasomporn and Sudsang \cite{sudsang_polygons},  Vahedi and van der Stappen \cite{vahedi}, and others.  A similar approach has also been adopted for caging 3D objects. For instance, Pipattanasomporn and Sudsang \cite{sudsang_polytopes} proposed an algorithm for computing all two-finger cages for non-convex polytopes.

In the above mentioned works fingertips are represented as points or spheres. Later, more complex shapes of caging tools were taken into account by Pokorny et al. \cite{pokorny, stork2013b, stork}, Varava \cite{varava}, Makita at al. \cite{makita, makita2}. In these works, sufficient conditions for caging were derived for objects with particular shape features.

In this paper, we address caging as a special case of path non-existence problem: an object is caged if there is no path leading it to an unbounded path-connected component.

The problem of proving path non-existence has been addressed by Basch et al. \cite{basch} in the context of motion planning, motivated by the fact most of the modern sampling-based planning algorithms do not guarantee that two configurations are disconnected, and rely on stopping heuristics in such situations \cite{latombe}.
Basch et al. prove that two configurations are disconnected when the object is `too big' or `too
long' to pass through a `gate' between them. In \cite{zhang}, Zhang et al. use approximate cell decomposition and prove path non-existence. They decompose a configuration space
into a set of cells and for each cell decide if it lies in the collision space. In \cite{mccarthy} McCarthy et al. propose a somewhat similar approach. There, they randomly sample the configuration space and reconstruct its approximation as an alpha complex. They later use it to check the connectivity between pairs of configurations. 

In this paper, we also aim to study path-connectivity of the free space of the object. Unlike \cite{mccarthy}, we do not construct the collision space directly. Instead we decompose it into a finite set of lower dimensional `slices'. This allows us to overcome the dimensionality problem without losing any necessary information about the topology of the configuration space.

\section{Definitions, Notation, and Overview}
\label{definitions}
Let us now provide the necessary definition and formulate our problem.
Since our work is related both to object manipulation and motion planning, we use the general term `object' without loss of generality when talking about both objects and autonomous rigid robots (e.g., disc robots) moving in two-dimensional workspaces.

\subsection{Definitions and Notation}
\begin{dfn}
\label{dfn::object}
A {\it rigid object} is a compact connected non-empty subset of $\mathbb{R}^2$.
A {\it set of obstacles}  is a compact  non-empty subset of $\mathbb{R}^2$.
\end{dfn}

We would like to limit ourselves to a large class of sufficiently `good' shapes: namely, we want both the object and the obstacles to be representable as a set of $2-$dimensional balls. Therefore, we do not allow them to have `thin parts'. To formalize this notion, we recall the definition of a regular set:

\begin{dfn}
\label{dfn::regular_set}
A set $U$ is {\it regular} if it is equal to the closure of its interior: 
$$U = \operatorname{cl}(\operatorname{int}(U))$$
\end{dfn}

In the above definition, the interior of $U$ is the largest open set contained in $U$, and the closure of $\operatorname{int}(U)$ is the smallest closed set containing $\operatorname{int}(U)$.

We make the following assumption:

\begin{asm}
Both the object and the obstacles  are regular sets.
\end{asm}

We approximate both the obstacles and the object as unions of balls lying in their interior, $\mathcal{S} = \{B_{R_1}(X_1), .., B_{R_n}(X_n)\}$  and $\mathcal{O} = \{O_{r_1}(Y_1), .., O_{r_m}(Y_m)\}$ of radii $R_1, .., R_n$ and $r_1, .., r_m$ respectively.

Let $\mathcal{C} = SE(2)$ denote the configuration space of the object. We define its {\it collision space} $\mathcal{C}^{col}$ as the set of the objects configurations in which the object penetrates the obstacles: 

\begin{dfn}
\label{dfn::collision_space}
$\mathcal{C}^{col} = \{c \in \mathcal{C}| \operatorname{int}(c(\mathcal{O})) \cap \operatorname{int}(\mathcal{S}) \neq \emptyset\}$, where $c(\mathcal{O})$ denotes the object in a configuration $c$. 
The {\it free space} $\mathcal{C}^{free}$ is the complement of the collision space: $\mathcal{C}^{free} = \mathcal{C} - \mathcal{C}^{col}$.
\end{dfn}

Note that this definition allows the object to be in contact with the obstacles.

\begin{dfn}
Two configurations $c_1$ and $c_2$ are called {\it path-connected} if there exists a continuous collision-free path $\gamma: [0, 1] \to \mathcal{C}^{free}$ between them: $\gamma(0) = c_1$, $\gamma(1) = c_2$.
If two configurations are path-connected, they belong to the same path-connected component.
\end{dfn}

\subsection{An Overview of Our Approach}
To compute path-connected component of the free space, we decompose the free space into a set of two-dimensional slices. The key idea of this approach has been discussed in our paper \cite{varava_2}. Now we provide the implementation details and use our approach to compute geometric characteristics of the free space introduced in the next section. 

In \cite{varava_2}, we suggested that configuration space decomposition might be a more computationally efficient alternative to its direct construction. Namely, we represent the configuration space as a product $\mathcal{C} = \R^{2} \times SO(2)$, and consider a finite covering of $SO(2)$:
$$
SO(2) = \bigcup_{i \in \{1,..,s\}}U_i
$$

Note that this is always possible, since $SO(2)$ is compact.

We recall the notion of a slice, introduced in \cite{varava_2}:
\begin{dfn}
A {\it slice} of the configuration space $\mathcal{C}$ is a subset of $\mathcal{C}$ defined as follows:
$$
Sl_U = \R^2 \times U,
$$
where $U$ is an subset of $SO(2)$.
\end{dfn} 

We denote a slice of the collision (free) space by $Sl_U^{col}$ ($Sl_U^{free}$). For each slice we construct an approximation $aSl_U^{col}$ of its collision space in such a way that our approximation lies inside the real collision space of the slice, 
$$aSl_U^{col} \subset Sl_U^{col}$$

This way, we approximate the entire collision space by its subset $a\mathcal{C}^{col}$: 

$$
a\mathcal{C}^{col} = \bigcup_{i \in \{1,..,s\}} aSl_{U_i}^{col} \subset \mathcal{C}^{col}
$$

In Sec.~\ref{our_approach}, we recall how we construct these approximations and provide details of our approach.

\section{Free Space and its Properties}
Let us now discuss how we can quantitatively estimate the `quality' of path-connected components and paths between different parts of the free space. We start with the volume of a component, which can be interesting for object manipulation. Then, we consider the width of narrow passages of the free space, which can be applied to motion planning.

\subsection{The Volume of Path-Components}
In robotic manipulation, it is useful to distinguish between big and small connected components, see Fig.~\ref{fig::amount-freedom}.
Recall that the free space is a subset of $SE(2) = \R^2 \times SO(2)$. There is a natural parametrization of $SO(2)$ as a one-dimensional manifold by the \emph{angle} $\theta\in\R$ and therefore we can represent an element of $SE(2)$ by $(x,y,\theta)$.

Now, note that since $SE(2)$ is a Lie group \cite{geomcontrol}, each point $p\in SE(2)$ gives rise to a diffeomorphism $L_p : SE(2)\to SE(2)$. Concretely, given a point $(x,y,\theta)$ in $SE(2)$, this function has the form:

\[ L_{(x',y',\theta')} (x,y,\theta) = (x' + x ,y' + y,\theta + \theta'), \]

which is a diffeomorphism of $SE(2)$ to itself. Furthermore, from this we can obtain the volume form of $SE(2)$ as the pullback of a volume element along $L_p$:

\[ \omega_p = dL_{-p}^*\omega_{(0,0,0)}\]

Therefore, choosing $\omega_{(0,0,0)} = dx\wedge dy\wedge d\theta$ we can obtain the volume of a subset of $SE(2)$ as:

\[ vol(U) = \int_U dx\wedge dy\wedge d\theta  = \int_U dx\,dy\,d\theta \]

Thus, for a set of the form $ U = [x_1,x_2]\times [y_1,y_2]\times [\theta_1,\theta_2]$ in $SE(2)$, we get $vol(U) = (x_2-x_1)(y_2-y_1)\min(\theta_2-\theta_1,2\pi)$.

\begin{figure}[htb!]   
\center{\includegraphics[width=0.25\textwidth]{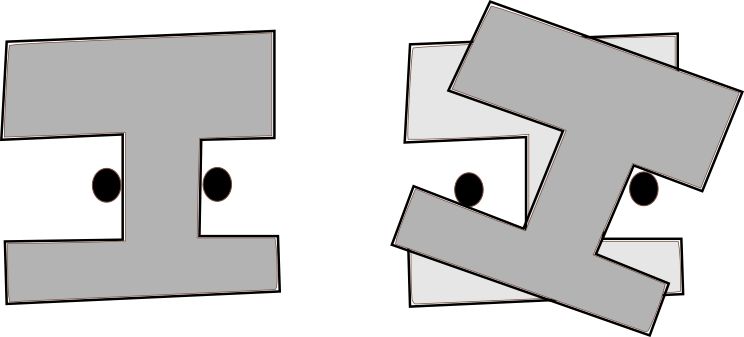}}
  \caption{On the left, an object is caged so that it can only move up and down; on the right, it can rotate within the cage, so the first cage provides a more accurate position of the object. In terms of the free space, in the second cage is larger than the first.}
  \label{fig::amount-freedom} 
  \vspace*{-.3cm}
\end{figure}

\subsection{The Width of Narrow Passages}
Motion planning becomes more challenging in presence of narrow passages. In particular, although our definition of collision space allows the object to be in contact with the obstacles, in practice we might want to avoid such situations. Furthermore, since we deal with approximations of objects and obstacles, the notion of robustness of a path is useful.  Therefore, another interesting geometric characteristic is the width of narrow passages of the free space. See Fig.~\ref{fig::narrow-pass}.

\begin{figure}[htb!]    
\center{\includegraphics[width=0.35\textwidth]{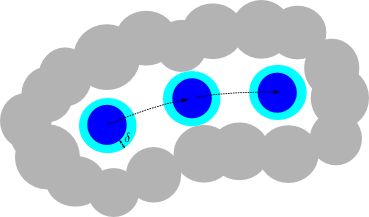}}
  \caption{The obstacles are depicted in grey, and the object is depicted in blue. The width of the narrow passage is equal to $\delta$ -- a $\delta$-offset of the object (depicted in light blue) can pass through it.}
  \label{fig::narrow-pass}
  \vspace*{-.5cm}
\end{figure}

Let $\dsign^O (x): \R^2 \to \R$ denote the signed distance from the point $x \in \R^2$ to the set $O$, defined as $\dsign^O (x) = \dist(x, \partial O)$ if $x \notin O$ and $\dsign^O (x) = - \dist(x, \partial O)$ if $x \in O$.

We say that two configurations are $\delta-$connected if there is a $\delta-$clearance path between them, or, more formally, we define $\delta-$connectivity as follows:

\begin{dfn}
Two collision-free configurations $c_1, c_2$ of $\mathcal{O}$ are {\it $\delta-$connected}, if there is a collision-free path $\gamma: [0, 1] \to \mathcal{C}^{free}$ between them such that the signed distance from any point of the object to the obstacles along this path is not less than $\delta$:
$$
\forall c \in \gamma, \
\forall x \in c(\mathcal{O}): \
\dsign^\mathcal{S} (x) \geq \delta
$$
\end{dfn}

Note that $\delta-$connectivity is a generalization of path-connectivity: two configurations are path-connected if and only if they are $0-$connected.

In the next section, we discuss how we check $\delta-$connectivity of the free space. 
\section{Free Space Decomposition}
\label{our_approach}
\subsection{An $\varepsilon-$core of the object}

First of all observe that by Def.~\ref{dfn::collision_space}, if a subset $\mathcal{O}'$ of an object $\mathcal{O}$ placed in configuration $c \in \mathcal{C}$ is in collision, then the entire object $\mathcal{O}$ is in collision. Therefore, the collision space of $\mathcal{O}$ is completely contained within the collision space of $\mathcal{O}'$. This allows us to make the following observation:
\begin{obs}
\label{obs::correctness}
Consider an object $\mathcal{O}$ and a set of obstacles $\mathcal{S}$. Let $c_1, c_2 \in \mathcal{C}^{free}$ be two collision-free configurations of the object. If there is no collision-free path between these configurations for its subset $\mathcal{O}' \subset \mathcal{O}$, then there is no collision-free path connecting these configurations for the original object.
\end{obs}
\begin{crl}
\label{crl::caging-correctness}
If some subset $\mathcal{O}'$ of $\mathcal{O}$ in the configuration $c$ is caged, then the entire object $\mathcal{O}$ in the same configuration $c$ is also caged.
\end{crl}

From the above observation it follows that if we construct $aSl_{U}^{col}$ in such a way that for any configuration $c \in aSl_{U}^{col}$ there exists a subset $\mathcal{O}'$ of $c(\mathcal{O})$ such that $\mathcal{O}'$ is in collision, then $c(\mathcal{O})$ is also in collision.

In \cite{varava_2} we defined an $\varepsilon-$core of an object as follows:

\begin{dfn}
An $\varepsilon$-core of an object $\mathcal{O}$ is a subset $\mathcal{O}_{\varepsilon}$ of the object such that any point of $\mathcal{O}_{\varepsilon}$ lies at least at a distance\footnote{By distance here we mean Euclidean distance in $\R^2$} $\varepsilon$ from the boundary of $\mathcal{O}$:
$$
\mathcal{O}_{\varepsilon} = \{p \in \mathcal{O}| d(p, \partial \mathcal{O}) \geq \varepsilon \} 
$$ 
\end{dfn}

Let us denote an object $\mathcal{O}$ and its $\varepsilon$-core $\mathcal{O}_{\varepsilon}$ by $\mathcal{O}^{\phi}$ and $\mathcal{O}_{\varepsilon}^{\phi}$ respectively when their orientation $\phi \in SO(2)$ is fixed. Let $\mathcal{C}^{col}(\mathcal{O}_{\varepsilon}^{\phi})$ denote the collision space of the $\varepsilon-$core with a fixed orientation $\phi$. Note that since the orientation is fixed, $\mathcal{C}^{col}(\mathcal{O}_{\varepsilon}^{\phi})$ is a subset of $\R^2$.

In \cite{varava_2}, we showed that for an object $\mathcal{O}^{\phi}$ and its $\varepsilon$-core $\mathcal{O}^{\phi}_{\varepsilon}$ with a fixed orientation $\phi$ there always exists a non-empty neighbourhood  $U(\phi) \subset SO(2)$ of $\phi$ such that for any $\theta \in U(\phi)$ the $\varepsilon$-core $\mathcal{O}^{\phi}_{\varepsilon}$ is fully contained within a slightly rotated object $\mathcal{O}^{\theta}$, see Fig.~\ref{rotated}.

\begin{wrapfigure}{r}{0.2\textwidth} 
\includegraphics[width=0.2\textwidth]{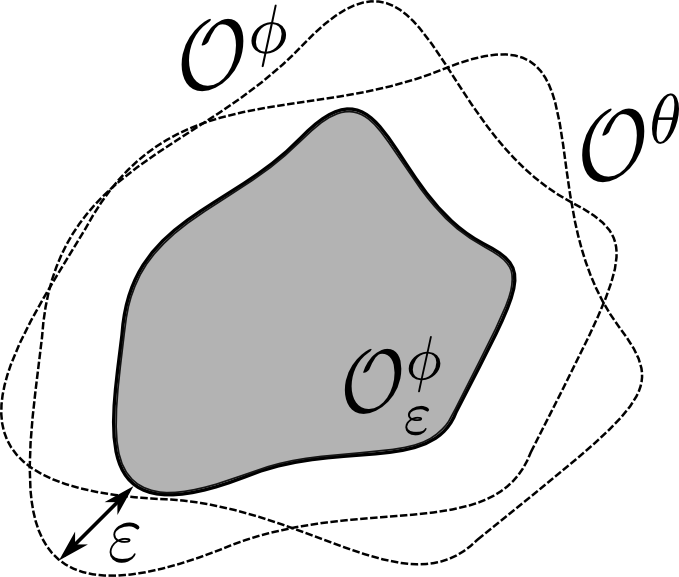}
    \caption{An $\varepsilon-$core remains inside the object when we slightly rotate it}
    \label{rotated}
\end{wrapfigure} 

So, we construct the collision space approximation as follows: we pick an $\varepsilon > 0$ and a sufficient set of orientation values $\{\phi_1, .., \phi_s\}$ so that $U(\phi_1), .. U(\phi_s)$ form a covering of $SO(2)$ such that for any $\theta \in U(\phi_i)$ we have $\mathcal{O}^{\phi_i}_{\varepsilon} \subset \mathcal{O}^{\theta}$.
This gives us a decomposition of the entire configuration space:
$$\mathcal{C} = \bigcup_{i \in \{1, .., s\}} Sl_{U(\phi_i)}$$

For each $\phi_i \in \{\phi_1, .., \phi_s \}$, we compute collision space slice approximations $aSl_U(\phi_i)^{col}$ as the collision space of $\mathcal{O}^{\phi_i}_{\varepsilon}$:
$$
aSl_U(\phi_i)^{col} = \mathcal{C}^{col}(\mathcal{O}_{\varepsilon}^{\phi_i}) \times U(\phi_i)
$$

\subsection{$SO(2)$ Partition}

Let $\D(\Delta \phi)$ denote the maximal displacement of any point $p \in \mathcal{O}$ after rotating it to an angle $\Delta \phi$ around the object's geometric center:
$$
\D(\Delta \phi) = \max_{p \in \mathcal{O}}(\dist(p, \operatorname{R}_{\Delta \phi} (p))),
$$
where $\operatorname{R}_{\Delta \phi} (p)$ is the rotation operator. Then $\mathcal{O}^{\theta}_{\varepsilon} \subset \mathcal{O}^{\theta}$ is fully contained inside a rotated object $\mathcal{O}^{\theta + \Delta \phi}$ if $\D (\Delta \phi) < \varepsilon$. In \cite{varava_2}, we have shown that 
$$
\D(\Delta \phi) \leq 2|\sin(\Delta \phi/2)|\cdot \diam(\mathcal{O}),
$$
assuming that we rotate the object around its geometric center and $\diam(\mathcal{O})$ denotes the biggest distance from it to any point of the object. Note that $\D(\Delta \phi)$ monotonically decreases with $\Delta \phi$.

So, given an $\varepsilon$ we pick $\Delta \phi$ such that 
$\D (\Delta \phi) < \varepsilon$, and compute a set of orientation samples $\{\phi_1 = 0, \phi_2 = \Delta \phi, .., \phi_{s+1} = \min\{2 \pi, s \cdot \Delta \phi\}\}$. This gives us a covering $P(\Delta \phi) = \{U_1(\phi_1), .., U_s(\phi_s)\}$ of $SO(2)$ of the form $U_i(\phi_i) = [\phi_i, \phi_{i+1}].$

Given $\varepsilon$ and $\Delta \phi$, we approximate the collision space by 
$$
a\mathcal{C}^{col}(\varepsilon, \Delta \phi) = \bigcup^{s}_{i = 1} \mathcal{C}^{col}(\mathcal{O}_{\varepsilon}^{\phi_i}) \times U_i.
$$

\subsection{Construction of Slices}

In \cite{varava_2}, we derive the following representation for the collision space of  $\mathcal{O}^{\phi_i}_{\varepsilon}$:

$$
\mathcal{C}^{col}(\mathcal{O}_{\varepsilon}^{\phi}) = \bigcup_{i \in \{1 .. m\}, \ j \in \{1 .. n\}} (B_{R_j + r_i - \varepsilon}(X_j - \overline{GY_i})),
$$

where $G$ is the origin chosen as the geometric center of the object, and $\overline{GY_i}$ denotes the vector from $G$ to $Y_i$.

For simplicity, in our implementation we assume that the obstacles and the object are approximated by sets balls of equal radii: $R_1 = .. = R_n = R$, $r_1 = .. = r_m = r$.

Now, let us discuss how we construct path-connected components of $\mathcal{C}^{col}(\mathcal{O}_{\varepsilon}^{\phi})$. Since $\mathcal{C}^{col}(\mathcal{O}_{\varepsilon}^{\phi})$ is a collection of balls, we represent it as an alpha complex. 

An alpha complex $A(R)$ representing a union of balls $\bigcup_{i = 1}^{n} B_R(x_i)$  is a simplicial complex with vertices $\{x_1, .., x_n\}$ which lies strictly inside $\bigcup_{i = 1}^{n} B_R(x_i)$. By the nerve theorem, an alpha complex is homotopy equivalent to the union of balls it approximates \cite{edelsbrunner}. 

Assume we have a finite set of points $X \subset \R^2$. Let us continuously increase the radius of the balls centered at these points, and consider a nested family of unions of balls. Correspondingly, we get a nested family of alpha complexes, $\emptyset = A(R_0) \subset A(R_1) \subset .. \subset D(X)$, where $D(X)$ is the Delaunay triangulation of $X$. Since $D(X)$ is finite, and any $A(R_i)$ is its subcomplex, the family of nested subcomplexes is also finite.

\vspace*{-.3cm}
\begin{algorithm}[htb]
\small
  \SetAlgoLined
  \KwData{A union of $d$-dimensional balls $\mathcal{C}^{col}(\mathcal{O}_{\varepsilon}^{\phi})$}
  \KwResult{A set of connected components ${aC_0, .., aC_n}$}
  $D(\mathcal{C}^{col}(\mathcal{O}_{\varepsilon}^{\phi})) \leftarrow $ Delaunay-Triangulation$(\mathcal{C}^{col}(\mathcal{O}_{\varepsilon}^{\phi}))$\\      
  $C_{0} \leftarrow$ Compute-Infinite-Component()\\
  $i \leftarrow 0$ \\
  \ForEach{$t_j \in $ Get-Exterior-Triangles$(D(\mathcal{C}^{col}(\mathcal{O}_{\varepsilon}^{\phi})), R+r-\varepsilon)$}{		
				\If{Marked($t_j$)}
				{$i$++\\
				$aC_{i} \leftarrow$ Compute-Component($t_j$)
				}		
	}
	
	\Return{$\{aC_0, .., aC_n\}$}
  \caption{Compute-Slice-Connectivity}
  \label{compute-slice-connectivity}
\end{algorithm} 
\vspace*{-.3cm}
Here by $aC_i \subset \R^2$ we denote a triangular representation of the three-dimensional path-connected component $C_i$ of a slice. These approximations are constructed as connected components of the complement of the alpha-complex representing the collision space of a slice.

{\it Delaunay-Triangulation()} computes the Delaunay triangulation based on the centers of the balls representing the collision space.

{\it Compute-Infinite-Component()} computes an artificially created connected component representing the unbounded connected component of the free space. It contains those exterior triangles which can be reached from the boundary of the Delaunay triangulation.

{\it Get-Exterior-Triangles$(D(\mathcal{C}^{col}(\mathcal{O}_{\varepsilon}^{\phi})), R+r-\varepsilon)$} returns a list of faces of the Delaunay triangulation lying outside of the alpha-complex $A(R+r-\varepsilon)$, which approximates the collision space represented as a union of balls of radius $R+r-\varepsilon$.

\subsection{The Connectivity Graph}

Once we have computed connected components of each slice, we can analyze the connectivity between slices. A detailed description of this procedure can be found in \cite{varava_2}, and we briefly recall it here. Let $\mathcal{G}(a\mathcal{C}^{col}(\varepsilon, \Delta \phi)) = (V, E)$ be a graph approximating the free space.  The vertices of $\mathcal{G}$ correspond to the connected components $\{aC^i_1, .., aC^i_{n_i}\}$ of each slice, $i \in \{1,..s\}$, and are denoted by $v = (aC, U)$, where $aC$ and $U$ are the corresponding component and orientation interval.

Two vertices representing components $C_p \subset aSl_{U_i}^{free}$ and $C_q \subset aSl_{U_j}^{free}$, $i \neq j$, are connected by an edge if the object can directly move between them. For that, the sets $U_i$ and $U_j$ must overlap: $U_i \cap U_j \neq \emptyset$,
 and the respective triangular approximations of path-connected components must intersect: 
 $aC_q \cap aC_p \neq \emptyset$.

$\mathcal{G}(a\mathcal{C}^{col}(\varepsilon, \Delta \phi))$ approximates the free space of the object: if two configurations are disconnected in $\mathcal{G}(a\mathcal{C}^{col}(\varepsilon, \Delta \phi))$, then they are disconnected in $\mathcal{C}^{free}$. Moreover, we later show that if two configurations are not path-connected in $\mathcal{C}^{free}$, then they are not $\delta-$connected in  $\mathcal{G}(a\mathcal{C}^{col}(\varepsilon, \Delta \phi))$ for any $\delta > \varepsilon$.

\subsection{The Width of the Narrow Passages}
Consider a superset of our object $\mathcal{O}$, defined as a set of points lying at most at distance $\delta$ from $\mathcal{O}$, and let us call it a $\delta-$offset of the object:
\begin{dfn}
\label{dfn::delta-offset}
A {\it $\delta-$offset} of an object $\mathcal{O}$ is defined as
$$
\mathcal{O}_{\delta} = \{p \in \R^2| d(p, \partial \mathcal{O}) \leq \delta \},
$$
where $d(.,.)$ denotes Euclidean distance in $\R^2$.
\end{dfn}

Now observe that by definition, two configurations are $\delta-$connected if and only if there exists a collision-free path connecting these configurations in the configuration space of the $\delta$-offset $\mathcal{O}_{\delta}$ of the object.
This means that two configurations $c_1$ and $c_2$ are not $\delta-$connected in $\mathcal{C}^{free}(\mathcal{O})$ if and only if they are not path-connected in $\mathcal{C}^{free}(\mathcal{O}_{\delta})$.

Therefore, to understand the $\delta-$connectivity of the free space it is enough to compute path-connected components as described in Alg.~\ref{compute-slice-connectivity} for the $\delta$-offset $\mathcal{O}_{\delta}$.

We can also slightly modify Alg.~\ref{compute-slice-connectivity} to find potential narrow passages for each $\delta > 0$. In this case, instead of constructing a Delaunay triangulation, we can compute a nested family of alpha-complexes of the centers of the collision balls for {\it all} positive alpha. As we mentioned before, this family is always finite. Increasing the $\alpha$ parameter corresponds to increasing the radius of the object under consideration, i.e., the family of alpha complexes approximates a nested family of collision space approximations of $\delta-$offsets of the object for all positive $\delta$. As we increase $\delta$, the topology of the free space of the slice will change, and the narrow passages will disappear. The corresponding values of $\delta$ reflect the width of the passages.
\section{Theoretical Guarantees of Our Approach}

In this section, we discuss correctness, completeness and computational complexity of our approach.

\subsection{Correctness}
First of all, let us show that our algorithm is correct: i.e., if there is no collision-free path between two configurations in our approximation of the free space, then these configurations are also disconnected in the actual free space. 

\begin{prp}
Consider an object $\mathcal{O}$ and a set of obstacles $\mathcal{S}$. Let $c_1, c_2$ be two collision-free configurations of the object. If $c_1$ and $c_2$ are not path-connected in $\mathcal{G}(a\mathcal{C}^{free}_{\varepsilon, \Delta \phi})$, then they are not path-connected in $\mathcal{C}^{free}$.
\end{prp}

\begin{proof}
Recall that the approximation of the free space is constructed as follows:
$$
a\mathcal{C}^{free}(\varepsilon, \Delta \phi) = \bigcup^{s}_{i = 1} aSl^{free}_{U(\phi_i)},
$$
where 
\begin{equation}
\label{eqn::free-slice-dfn}
aSl^{free}_{U(\phi_i)} = Sl_{U(\phi_i)} - aSl^{col}_{U(\phi_i)},
\end{equation}
and 
\begin{equation}
\label{eqn::col-slice-dfn}
aSl^{col}_{U(\phi_i)} = \mathcal{C}^{col}(\mathcal{O}_{\varepsilon}^{\phi_i}) \times U(\phi_i)
\end{equation}

Let us first recall that by construction for any $\phi \in U(\phi_i)$ we have $\mathcal{O}_{\varepsilon}^{\phi_i} \subset \mathcal{O}^{\phi}$, and therefore
$\mathcal{C}^{col}(\mathcal{O}_{\varepsilon}^{\phi_i}) \subset \mathcal{C}^{col}(\mathcal{O}^{\phi})$.

Therefore, $aSl^{col}_{U(\phi_i)} \subset Sl^{col}_{U(\phi_i)}$, which together with (\ref{eqn::free-slice-dfn}) implies that
\begin{equation}
Sl^{free}_{U(\phi_i)} \subseteq aSl^{free}_{U(\phi_i)},
\end{equation}
We now want to show that if there is no path between two vertices $v = (aC, U)$ and $v' = (aC', U')$ in $\mathcal{G}(a\mathcal{C}^{free}_{\varepsilon, \Delta \phi})$, then there is no path between connected components of $a\mathcal{C}^{free}(\varepsilon, \Delta \phi)$ corresponding to them. It is enough to show that if two vertices corresponding to adjacent slices are not connected by an edge, then they represent two components which are disconnected in the union of these adjacent slices.

Consider two adjacent slices $Sl_{U(\phi_i)}$ and $Sl_{U(\phi_{i+1})}$, and two path-connected components $C_1 \subset aSl^{free}_{U(\phi_i)}$ and $C_2 \subset aSl^{free}_{U(\phi_i)}$. Let $aC_1$ and $aC_2$ respectively be their triangular representations.

Let $v_1$ and $v_2$ be two vertices of $\mathcal{G}(a\mathcal{C}^{free}_{\varepsilon, \Delta \phi})$ corresponding to these components: $v_1 = (aC_1, U(\phi_i))$ and $v_2 = (aC_2, U(\phi_{i+1}))$. Let us show that if there is no edge between $v_1$ and $v_2$, then there is no path between the corresponding components $C_1$ and $C_2$ in $aSl_{U(\phi_{i})}^{free} \cup aSl_{U(\phi_{i+1})}^{free}$.
Indeed, since an alpha complex representing the collision space of a slice is a subset of its actual collision space, the complement of the alpha complex is a superset of the actual free space of the slice. So,  $C_1 \subset aC_1 \times U(\phi_i)$ and $C_2 \subset aC_2 \times U(\phi_{i+1})$. Now, if there is no edge between $v_1$ and $v_2$, then $aC_1 \cap aC_2 = \emptyset$. This implies that $aC_1 \times U(\phi_i) \cap aC_2 \times U(\phi_{i+1}) = \emptyset$, and therefore $C_1$ and $C_2$ are two disjoint components of $aSl_{U(\phi_{i})}^{free} \cup aSl_{U(\phi_{i+1})}^{free}$.

\end{proof}

\subsection{$\delta$-Completeness}
We would like to show that if two configurations are not path-connected in $\mathcal{C}^{free}$, we can always construct an approximation of $\mathcal{C}^{free}$ in which these configurations are either disconnected or connected by a narrow passage.

\begin{prp}
Let $c_1, c_2$ be two configurations in $\mathcal{C}^{free}$. If they are not path-connected in $\mathcal{C}^{free}$, then for any $\varepsilon > 0$ and $\delta > \varepsilon$ there exists $\Delta \phi > 0$ such that the corresponding configurations are not $\delta-$connected in the approximated free space $a\mathcal{C}^{free}_{\varepsilon, \Delta \phi}(\mathcal{O})$. 
\end{prp}
\begin{proof}
First, observe that if two configurations $c_1$ and $c_2$ are not path-connected in $\mathcal{C}^{free}(\mathcal{O})$, then for any $\delta > \varepsilon$ they are not $\delta-$connected in $\mathcal{C}^{free}(\mathcal{O}_{\varepsilon})$. We want to show that they are not $\delta-$connected in $a\mathcal{C}^{free}_{\varepsilon, \Delta \phi}(\mathcal{O})$ for a particular choice of $\Delta \phi$. Let $\Delta \phi$ be such that 
\begin{equation}
\label{eqn::delta-phi}
\D(\Delta \phi) < \delta - \varepsilon,
\end{equation}

where $\D(\Delta \phi)$ is the displacement function.

Since $c_1$ and $c_2$ are not path-connected there exists a collision configuration $c$ in any path between them, so, $\dsign^{\mathcal{S}}(c(\mathcal{O})) < 0$. Therefore, for the same configuration $c$ we have $\dsign^{\mathcal{S}}(c(\mathcal{O}_{\varepsilon})) < \varepsilon$. Let $i \in \{1, .., s\}$ be such that $\phi \in U_i(\phi_i)$, and let $c' = (x, \phi_i)$. We get
\begin{multline*}
\dsign^{\mathcal{S}}(c'(\mathcal{O}_{\varepsilon})) \leq \dist(c(\mathcal{O}_{\varepsilon}), c'(\mathcal{O}_{\varepsilon})) + \dsign^{\mathcal{S}}(c(\mathcal{O}_{\varepsilon})) \leq \\
  \dist(c(\mathcal{O}_{\varepsilon}), c'(\mathcal{O}_{\varepsilon})) + \varepsilon \leq \D(\Delta \phi) + \varepsilon < \delta.
\end{multline*}

\end{proof}

\subsection{Computational Complexity}

Let us now estimate the computational complexity of our approach. 
Let $n$ and $m$ be the number of balls in the obstacles and the object approximations, respectively, and let $s$ be the number of slices we need. Let us focus on the basic version of our algorithm when we are not interested in $\delta-$connectivity of the space, and limit ourselves to path-connectivity. Our algorithm has two major steps: first, we compute approximations of the slices and their path-connected component, and then we construct the connectivity graph. For each slice, we compute a Delaunay triangulation of the union $n \, m$ balls, pick $\alpha = (R + r - \varepsilon)^2$, and compute the path-connected components of the corresponding alpha-complex. The computation of path-connected components is linear on the number of faces of the complex, and therefore the overall complexity of this step is dominated by the computation of Delaunay triangulation, which is $O((n \, m) \log(n \, m))$ \cite{guibas}.

Since we need to compute $s$ slices, the overall complexity of the first step of the algorithm is $O(s \, (n \, m) \log(n \, m))$.

Once the slices are computed, we proceed to the next step -- the connectivity graph construction. Assume that each slice has at most $q$ path-connected components. We create vertices of the connectivity graph by iterating through all the components of each slice, which gives us at most $O(s \, q)$ iterations. Then for each pair of overlapping slices we check whether their path-connected components intersect. Assume that each slice has at most $t$ faces in its alpha-complex representation. Given our partition of $SO(2)$, each slice has 2 neighbours, so the overall complexity of the edges construction is at most $O(2 \, s \, t^2)$. Note that this is a conservative estimation. In our current implementation, we check whether two components from different slices intersect by iterating over all the triangles in their triangulations, which makes the complexity of this procedure quadratic on the number of triangles. As a consequence, this step takes the majority of the overall computation time of the algorithm. However, this can be significantly improved using, for instance, quadtrees.

Finally, we compute connected components of the connectivity graph, The complexity of this procedure is $O(s \, q^2)$,  since each slice has 2 neighbours, and hence each vertex of the connectivity graph has at most $2 \, q$ adjacent vertices.

This gives us the overall complexity $O(s \, (n \,  m)\, \log(n \, m)) + O(s \, q) + O(2 \, s \, t^2) + O(s \, q^2)$. In practice, the number of connected components per slice is small unless the original space has a lot of path-connected components. So, the most computationally expensive parts are slices construction  -- $O(s \, (n \, m)\, \log(n \, m))$, and the connectivity graph edges construction -- $O(2\, s \, t^2)$. Note that both these parts can be parallelized. Each slice approximation can be computed separately, because the slices do not depend on each other. For each pair of adjacent slices we can compute the intersections between their connected components independently. 

In Sec.~\ref{our_approach} we said that the partition of $SO(2)$ is constructed in such a way that the distance $\Delta \phi$ between two consecutive orientation samples should be chosen such that the maximal displacement does not exceed the chosen $\varepsilon$. Note that $\varepsilon$ should be less than the radius $r$ of the smallest ball in the object's representation in order to preserve the shape of the object. On the other hand, Fig.~\ref{plot} shows that the number $s$ grows significantly when we decrease the $\varepsilon$. Finally, the possible choice of $\varepsilon$ depends on the shape of the object: if the shape is easy to approximate with a small number of large balls, then we can choose a large $\varepsilon$ without losing any crucial information about the shape. In contrast, if the object has a lot of thin parts, it has to be approximated with a larger number of smaller balls. In this case we have to choose a smaller $\varepsilon$, which increases the necessary number of slices.

\begin{figure}[htb!]   

\center{\includegraphics[width=0.4\textwidth]{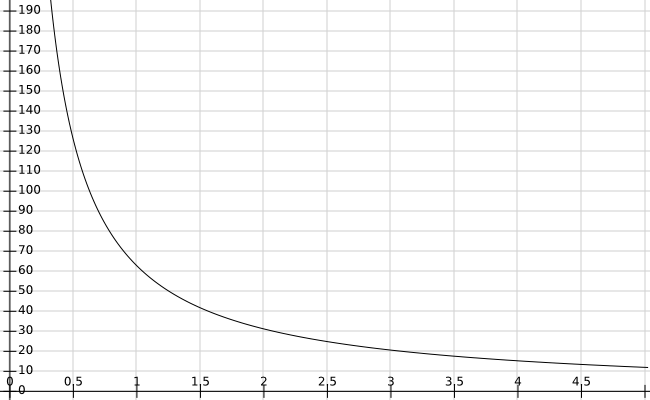}}
  \caption{This plot shows how the number of slices (Y-axis) depends on the $\varepsilon$ (X-axis) given an object of diameter 5.}
  \label{plot} 
  \vspace*{-.4cm}
\end{figure}

\section{Experiments and Results}
\label{experiments}

\begin{figure}
  \subcaptionbox{Workspace}[.45\linewidth]{%
    \includegraphics[width=.45\linewidth]{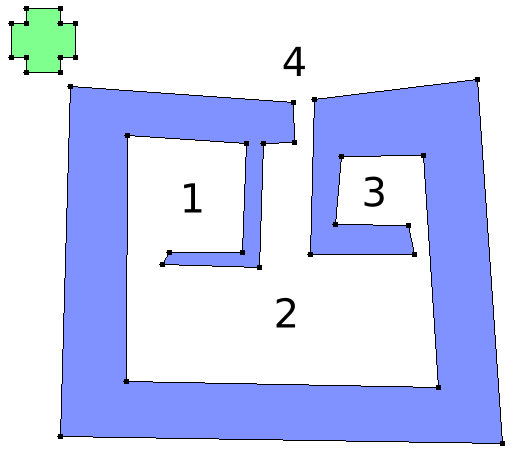}}\quad
  \subcaptionbox{Slice for $\phi = 0.1$}[.45\linewidth]{%
    \includegraphics[width=.45\linewidth]{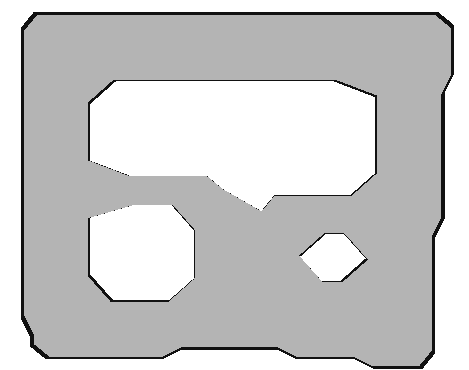}}
      \subcaptionbox{Slice for $\phi = 3.7$}[.45\linewidth]{%
    \includegraphics[width=.45\linewidth]{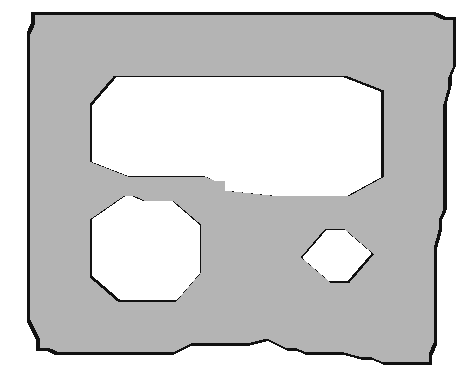}}\quad
          \subcaptionbox{Slice for $\phi = 6.2$}[.45\linewidth]{%
    \includegraphics[width=.45\linewidth]{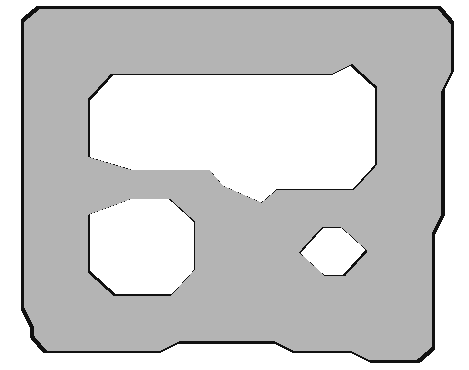}}
      \caption{On the first figure the obstacle is depicted in blue, and the object is depicted in green. There are 3 narrow passages separating 4 potential disjoint path-connected components: 1, 2, 3, and 4. The fourth component is unbounded. Figures 2, 3, and 4 depict alpha-complex approximations of different slices visualized in MeshLab. }
      \label{fig::workspace}
\end{figure}

\begin{figure*}[t]
  \centering
  \subcaptionbox{Radius = 15, 354 balls}[.3\linewidth][c]{%
    \includegraphics[width=.3\linewidth]{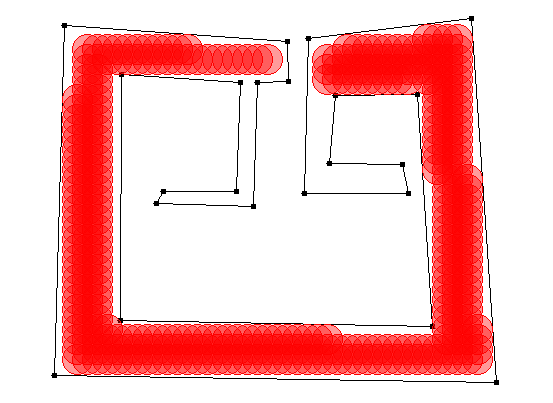}}\quad
  \subcaptionbox{Radius = 10, 492 balls}[.3\linewidth][c]{%
    \includegraphics[width=.3\linewidth]{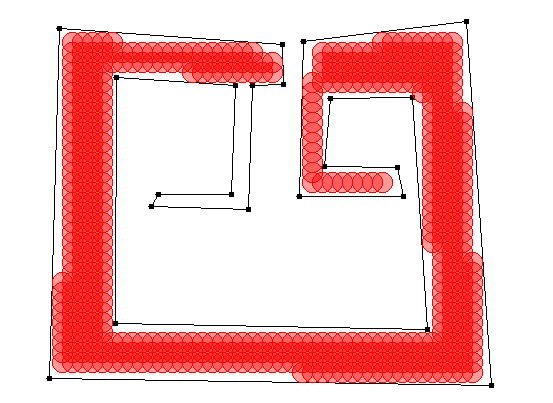}}\quad
  \subcaptionbox{Radius = 4, 681 balls}[.3\linewidth][c]{%
    \includegraphics[width=.3\linewidth]{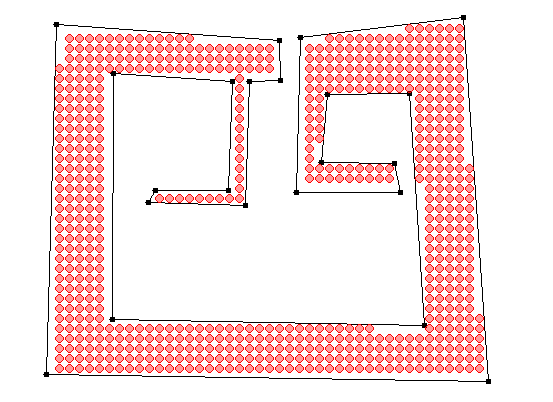}}
    \caption{We approximate the obstacles by sets of balls of radius 15, 10, and 4, respectively. Note that the first approximation significantly simplifies the shape, and has only one narrow passage; the second approximation preserves the shape better and has two narrow passages; the third approximation preserves all the important shape features of the obstacles.}
    \vspace*{-.2cm}	
    \label{fig::exp_1}
\end{figure*}

We consider a simple environment, see Fig.~\ref{fig::workspace}. For our experiments, we generate a workspace and an object as  polygons, and approximate them with unions of balls of equal radii lying strictly inside the polygons. Note that the choice of the radius is important: when it is small, we get more balls, which increases the computation time of our algorithm; on the other hand, when the radius is too large, we lose some important information about the shape of the obstacles, because thin parts cannot be approximated by large balls, see Fig.~\ref{fig::exp_1}.

We use CGAL library to compute Delaunay triangulation and alpha complexes. Our experiments were run on an Intel Core i7 laptop with 12 Gb RAM.

We consider a simple object whose approximation consists of 5 balls. We run our algorithm for all the 3 approximations of the workspace, and take 5 different values of $\varepsilon$, see Table~\ref{time}. We can observe that as we increase the $\varepsilon$ the computation time decreases. This happens because we are using fewer slices. However, we can also observe that when the $\varepsilon$ is too large, our approximation of the collision space becomes too small, and we are not able to find one connected component (see the last column of the table).

\begin{table}[htb!]
\captionsetup{font=small}
\vspace*{-.3cm}	
\small
\begin{tabular}{| c | c | c | c |}
  \hline			
    & $R = 15$ & $R = 10$ & $R = 4$   \\

  \hline
   $\varepsilon = 0.30  \cdot r$ & 2 c.; 741 ms & 3 c.;  1287 ms & 4 c.; 1354 ms  \\
   $\varepsilon = 0.33  \cdot r$ & 2 c.; 688 ms & 3 c.; 1208 ms &  4 c.; 1363 ms \\
   $\varepsilon = 0.37 \cdot r$ & 2 c.; 647 ms & 3 c.; 1079 ms & 4 c.; 1287 ms  \\
   $\varepsilon = 0.40 \cdot r$ & 2 c.; 571 ms & 3 c.; 986 ms &  3 c.; 1156 ms \\
   $\varepsilon = 0.43 \cdot r$ & 2 c.; 554 ms & 3 c.;  950 ms & 3 c.; 1203 ms \\
  \hline  
\end{tabular}

\caption{We run our experiments for 5 values of $\varepsilon$ and 3 workspace approximations. We report the number of path-connected components we found and the computation time for each case.}
\label{time}
\vspace*{-.1cm}
\end{table}

When we were using our first approximation of the workspace, we were able to distinguish only between components 4 and 2 (see Fig.~\ref{fig::exp_1}), and therefore prove path non-existence between them. For a more accurate approximation, we were also able to detect component 3. Finally, the third approximation of the workspace allows us to prove path non-existence between every pair of the four components. The accuracy of a workspace approximation depends on the task: for instance, if the only thing we need to know is whether the object can escape arbitrarily far from the obstacles, then it is enough to use a rough approximation. 

\begin{table}[htb!]
\centering
\captionsetup{font=small}
\vspace*{-.3cm}	
\small
\begin{tabular}{| c | c | c | c |}
  \hline			
    & Component 1 & Component 2 & Component 3   \\

  \hline
   $\varepsilon = 0.30  \cdot r$ & 68676 & 216685 & 18893  \\
   $\varepsilon = 0.33  \cdot r$ & 70527 & 223110  &  19934 \\
   $\varepsilon = 0.37 \cdot r$ & 71484 & 229040 & 20354  \\
   $\varepsilon = 0.40 \cdot r$ & -- & 307839 &  20841 \\
   $\varepsilon = 0.43 \cdot r$ & -- & 316281 & 21092 \\
  \hline  
\end{tabular}

\caption{This table reports the volume of components 1, 2, and 3. Component 4 is infinite. We see that when $\varepsilon$ is too large, our approximation of the collision space becomes too conservative, and the first 2 components merge.}
\label{volume}
\vspace*{-.4cm}

\end{table}

We have also computed an estimation of the volume of the 3 bounded components, see Table~\ref{volume}. For that, we used the most accurate approximation of the workspace, and 5 different values of $\varepsilon$. We can see that when we increase $\varepsilon$, the size of the components increases, because larger values of $\varepsilon$ provide more conservative (larger) approximations of the free space. For the last 2 values of $\varepsilon$ components 1 and 2 merge, and hence the volume of component 2 increases significantly.

\section{Conclusion}
In this paper, we propose an approach towards proving caging and path non-existence for rigid objects in 2D workspaces. We compute an approximation of the collision space of the object, represent it as a collection of lower dimensional projections, and analyze path-connectivity of the free space of the object. Apart from that, we estimate the volume of path-connected components and the width of narrow passages. We perform an experimental evaluation of our approach and show that our algorithm is correct and $\delta-$complete.


\begin{thebibliography}{99}
\bibitem{barraquand} Barraquand, J., Kavraki, L., Latombe, J.-C., Motwani, R.,  Li, T.-Y., Raghavan, P.: A random sampling scheme for path planning. In: The International Journal of Robotics Research, 16(6), 759–774 (1997).

\bibitem{basch} Basch, J., Guibas, L. J., Hsu, D., Nguyen, A. T.: Disconnection proofs for motion planning. In:  IEEE International Conference on Robotics and Automation (2001),  1765-1772. 

\bibitem{geomcontrol} Bullo, F., Lewis, A. D.: Geometric Control of Mechanical Systems. Springer, 2005.

\bibitem{edelsbrunner} Edelsbrunner, H., Harer, J.: Computational topology: an introduction. American Mathematical Soc., 2010.

\bibitem{guibas} Guibas,  L.J., Stolfi, J.: Primitives for the manipulation of general subdivisions and the computation of Voronoi diagrams. In: ACM Trans. Graphics, 4 (1985), 74-123.

\bibitem{kuperberg} Kuperberg, W.: Problems on polytopes and convex sets. In: DIMACS Workshop on polytopes (1990),  584-589.


\bibitem{latombe} Latombe, J.-C.: Robot Motion Planning. Norwell, MA, USA: Kluwer Academic Publishers (1991).

\bibitem{mahler} Mahler, J., Pokorny, F. T., McCarthy, Z.,  van der Stappen, A. F., Goldberg, K.: Energy-bounded caging: Formal definition and 2-D energy lower bound algorithm based on weighted alpha shapes. In: IEEE Robotics and Automation Letters, 1(1), 508-515 (2016).

\bibitem{makita} Makita, S., Maeda, Y.: 3D multifingered caging: Basic formulation and planning. In: IEEE Intelligent Robots and Systems (2008),  2697--2702. 
  

\bibitem{makita2}Makita, S., Okita, K., Maeda, Y.: 3D two-fingered caging for two types of objects: sufficient conditions and planning. In: International Journal of Mechatronics and Automation, 3(4),  263-277 (2013)

\bibitem{mccarthy} McCarthy, Z., Bretl, T., Hutchinson, S.: Proving path non-existence using sampling and alpha shapes. In: IEEE International Conference on Robotics and Automation (2012), 2563-2569.

\bibitem{sudsang_polygons}Pipattanasomporn, P., Sudsang, A.: Two-finger caging of concave polygon. In: IEEE International Conference on Robotics and Automation (2006),  2137--2142.

\bibitem{sudsang_polytopes} Pipattanasomporn, P., Sudsang, A.: Two-finger caging of nonconvex polytopes. In: IEEE Transactions on Robotics, 27(2), 324-333 (2011).


\bibitem{pokorny} Pokorny, F. T., Stork, J. A., Kragic, D.: Grasping objects with holes: A topological approach. In: IEEE International Conference on Robotics and Automation (2013), 1100--1107.


\bibitem{rimon} Rimon, E., Blake, A.: Caging planar bodies by one-parameter two-fingered gripping systems. In: The International Journal of Robotics Research, 18(3), 299--318 (1999).


\bibitem{rodriguez} Rodriguez, A., Mason, M. T., Ferry, S.: From caging to grasping. In: The International Journal of Robotics Research, 31(7), 886-900 (2012).

\bibitem{stork2013b} Stork, J. A., Pokorny, F. T., Kragic, D.: Integrated Motion and Clasp Planning with Virtual Linking. In: IEEE/RSJ International Conference on Intelligent Robots and Systems (2013), 3007-3014.


\bibitem{stork} Stork, J. A., Pokorny, F. T., Kragic, D.: A Topology-based Object Representation for Clasping, Latching and Hooking. In: IEEE-RAS International Conference on Humanoid Robots (2013), 138-145.

\bibitem{pereira} Pereira, G. A. S., Campos, M.F.M., Kumar, V.: Decentralized algorithms for multi-robot manipulation via caging. In: The International Journal of Robotics Research 23(7-8), 783 -- 795 (2004).


\bibitem{varava} Varava, A., Kragic, D., Pokorny, F. T.: Caging Grasps of Rigid and Partially Deformable 3-D Objects With Double Fork and Neck Features. In IEEE Transactions on Robotics, 32(6), 1479-1497 (2016).


\bibitem{varava_2} Varava, A., Carvalho, J. F., Pokorny, F. T., Kragic, D.: Caging and Path Non-Existence: a Deterministic Sampling-Based Verification Algorithm. International Symposium on Robotics Research, 2017 (accepted, to appear). 
Preprint: \url{http://www.csc.kth.se/~jfpbdc/path_non_ex.pdf}

\bibitem{vahedi} Vahedi, M., van der Stappen, A. F.: Caging polygons with two and three fingers. In: The International Journal of Robotics Research, 27(11-12) 1308--1324 (2008).


\bibitem{zhang} Zhang, L., Young, J. K., Manocha, D.: Efficient cell labelling and path non-existence computation using C-obstacle query. In: The International Journal of Robotics Research, 27(11-12), 1246-1257 (2008).

\end{thebibliography}
\end{document}